\documentclass[12pt]{article}

\usepackage{graphicx}
\usepackage{pgfplots}

\usepackage{amsmath,amssymb,amsthm,enumitem,mathtools, bbm}
\newtheorem{theorem}{Theorem}

\usepackage{multirow}
\usepackage{algorithm}% http://ctan.org/pkg/algorithms
\usepackage[noend]{algpseudocode} % Avoid "end" and make it look cleaner
\usepackage{parskip}

\usepackage{csvsimple}	% reading CSV files in tables
\usepackage{booktabs}   % Nicer Tables
\usepackage{adjustbox}  % To adjust table length

\usepackage[colorlinks=true, citecolor=cyan, linkcolor=cyan, urlcolor=cyan]{hyperref}

\renewenvironment{abstract}
 {\small
  \begin{center}
  \bfseries \abstractname\vspace{-1em}\vspace{0pt}
  \end{center}
  \list{}{
    \setlength{\leftmargin}{.2cm}%
    \setlength{\rightmargin}{\leftmargin}%
  }%
  \item\relax}
 {\endlist}

\title{Challenges in Variable Importance Ranking Under Correlation}

\author{Annie Liang\textsuperscript{1\dag}, 
Thomas Jemielita\textsuperscript{2\dag},\\
Andy Liaw\textsuperscript{2},
Vladimir Svetnik\textsuperscript{2}, 
Lingkang Huang\textsuperscript{2},\\
Richard Baumgartner\textsuperscript{2},
Jason M. Klusowski\textsuperscript{1}}

\date{\footnotesize
\textsuperscript{\textbf{1}} Department of Operations Research and Financial Engineering, Princeton University \hspace{2em} \\
\textsuperscript{\textbf{2}} Biostatistics and Research Decision Sciences, Merck \& Co., Inc., Rahway, NJ, USA\\
\textsuperscript{\textbf{\dag}}These authors contributed equally to this work.}

\begin{document}
\maketitle

\begin{abstract}
Variable importance plays a pivotal role in interpretable machine learning as it helps measure the impact of factors on the output of the prediction model. Model agnostic methods based on the generation of “null” features via permutation (or related approaches) can be applied. Such analysis is often utilized in pharmaceutical applications due to its ability to interpret black-box models, including tree-based ensembles. A major challenge and significant confounder in variable importance estimation however is the presence of between-feature correlation. Recently, several adjustments to marginal permutation utilizing feature knockoffs were proposed to address this issue, such as the variable importance measure known as conditional predictive impact (CPI). Assessment and evaluation of such approaches is the focus of our work. We first present a comprehensive simulation study investigating the impact of feature correlation on the assessment of variable importance. We then theoretically prove the limitation that highly correlated features pose for the CPI through the knockoff construction. While we expect that there is always no correlation between knockoff variables and its corresponding predictor variables, we prove that the correlation increases linearly beyond a certain correlation threshold between the predictor variables. Our findings emphasize the absence of free lunch when dealing with high feature correlation, as well as the necessity of understanding the utility and limitations behind methods in variable importance estimation.
\end{abstract}

\newpage
\tableofcontents
\newpage

\section{Introduction}
Estimating variable importance is a crucial aspect in the field of interpretable machine learning that has recently faced scrutiny \cite{watson2021},\cite{hooker2021}. The between feature correlation has been identified as a strong confounder and a main challenge in variable importance estimation \cite{strobl}, \cite{gregorutti}. For example, in \cite{hooker2021} extrapolation bias that artificially inflates permutation-based variable importance under feature correlation has been investigated. In this work, it was shown that an adjustment to accommodate the feature correlation is desirable. Recent developments in the knock-off literature have been utilized to propose a variable importance measure known as conditional predictive impact (CPI) \cite{watson2021}. This measure is based on the knock-off generation, which inherently deals with feature correlation \cite{candes}. In our work we focus on the limitations that highly correlated features pose for the CPI. 

Our manuscript is organized as follows: in the first part, we conduct a simulation study that highlights the limitations of traditional variable importance measures and the absence of a free lunch for feature correlation in the CPI. In the second part, we theoretically prove such limitation of the CPI by examining knockoff construction in the bivariate case.

\section{Methods}

Let $O = (Y, \textbf{X})$ denote the observed data with $i=1,..,n$ samples where $Y=[Y_1,...,Y_i,...,Y_n]$ is the outcome of interest and $\textbf{X}=[X_1,...,X_j,...,X_p]$ is a $p$ dimensional covariate space with $X_j = [X_{1,j},...,X_{i,j},...,X_{n,j}]$. Without loss of generality, we assume that each $X_j$ is standardized. In our evaluation, the below variable importance approaches were explored. 

\begin{enumerate}
    \item \textbf{Ordinary Least Squares}: Regress $Y \sim \beta \textbf{X}$ where $\beta=[\beta_1,...,\beta_p]$ corresponds to the variable importance estimates. Standard-errors and p-values are directly obtained via maximum likelihood estimation. 
    
    \item \textbf{Marginal Permutation-based Importance}: Let $L(f, X)$ denote some loss-function (ex: mean squared error, should be an out-of-sample estimate via cross validation or out-of-bag) for model $f$ and input $X$. Additionally, let $\Ddot{X}^{S}$ denote the permuted covariate space where $S$ indexes the variables being permuted. While this can involve permuting multiple variables, the marginal importance of variable $X_j$ is obtained by randomly shuffling $X_j$ only. Variable importance for variable $j$ can then be calculated as:
    \begin{eqnarray*}
        VI^{P}_j &= & L(f,\Ddot{\textbf{X}}^{j})-L(f,\textbf{X}) 
    \end{eqnarray*}
    where $VI^{P}$ denotes permutation-based variable importance. For ``important'' variables, $VI^{P}>0$ while for ``non-important'' variables $VI^{P}\leq 0$. Obtaining standard-errors and p-values generally require some type of resampling method such as bootstrapping. 
    
    \item \textbf{Conditional Predictive Impact (CPI)}: The Conditional predictive impact (CPI) was introduced in \cite{watson2021}. Similar to permutation-based variable importance, this involves creating ``fake'' samples of the covariate space. Importantly, CPI uses the knockoffs framework \cite{candes} to generate ``fake'' samples, which better accounts for potential correlation between covariates by using some type of knockoff generator that ideally closely mimics the distribution of $\textbf{X}$. For example, if $\textbf{X} \sim N(\mu, \Sigma)$, then an appropriate knockoff generator would involve estimating $\mu$ and $\Sigma$, then generating fake samples based on these estimates.
    
    Let $\tilde{X}^{S}$ denote the knockoff covariate space where $S$ indexes the variables with knockoffs. Variable importance for variable $j$ can then be calculated as:
    \begin{eqnarray*}
        VI^{CPI}_j &= & L(f,\tilde{\textbf{X}}^{j})-L(f,\textbf{X}) 
    \end{eqnarray*}
    
    As with permutation importance, ``important'' variables have $VI^{CPI}>0$ while for ``non-important'' variables have $VI^{CPI}\leq 0$. Standard errors and p-values can be obtained through wald type statistics or other related approaches as described in \cite{watson2021}. 
\end{enumerate}

\section{Simulation Study}

For our simulation study, we followed the data generation according to \cite{hooker2021}. The target $f(X)$ is given as $f(X)=\beta_0+\sum_{j=1}^p\beta_jx_j$ such that the targets $y_i$ were generated as $y_i=f(X_i)+\epsilon_i$ with $X_i \sim Uniform(0,1)$ and $\epsilon_i \sim N(0,\sqrt{0.1})$.

\subsection{Simulation Scenarios}

The following scenarios for different $f(X_i)$ were investigated:
\begin{enumerate}
    \item \textbf{Scenario 1}, the original \cite{hooker2021} scenario:
\begin{equation*}
f(X_{i})=x_{i1}+x_{i2}+x_{i3}+x_{i4}+x_{i5}+0x_{i6}+0.5x_{i7}+0.8x_{i8}+1.2x_{i9}+1.5x_{i10}
\end{equation*}
Here, the variables (features) $x_{i1},\ldots,x_{i5}$ are control variables with the same importance, $x_{i6}$ is a null variable, $x_{i7},x_{i8}$ are lower signal variables, and $x_{i9},x_{i10}$ represent higher signal variables. The correlation between $x_{i1},x_{i2}$ was allowed to vary and were generated via Gaussian copula \cite{hooker2021},\cite{nelsen2007}.
    \item \textbf{Scenario 2}:
\begin{equation*}
f(X_{i})=0x_{i1}+x_{i2}+x_{i3}+x_{i4}+x_{i5}+0x_{i6}+0.5x_{i7}+0.8x_{i8}+1.2x_{i9}+1.5x_{i10}
\end{equation*}
Same as scenario 1, but the first variable is set to be a ``null,'' non-signal variable.
    \item \textbf{Scenario 3}:
    \begin{equation*}
f(X_{i})=(x_{i1}+x_{i2})/2+x_{i3}+x_{i4}+x_{i5}+0x_{i6}+0.5x_{i7}+0.8x_{i8}+1.2x_{i9}+1.5x_{i10}
    \end{equation*}
Same as scenario 1, but only the mean of the two correlated variables is observed ($(x_{i1}+x_{i2})/2$).
\item \textbf{Scenario 4}:
Same as scenario 2, but only the mean of the first two correlated variables is observed ($(x_{i1}+x_{i2})/2$).
\end{enumerate}

For each scenario, $corr(x_{i1},x_{i2})=\rho \geq 0$ was allowed to vary. For each scenario / correlation setting, 1000 simulated data-sets were generated to assess variable importance metrics. The key metric of interest was the estimated rank for each feature. For a given model, this involves calculating the variable importance for each feature and then estimating the relevant ranking. For a simple linear model, this corresponds to the rank of the estimated beta coefficients.  

\subsection{Simulation Results}

First, variable importance rankings are evaluated across the considered scenarios and models. See Figure \ref{fig:rank_scen1_scen2} and Figure \ref{fig:rank_scen3_scen4} for the average importance ranking for Features $X_1$, $X_2$, and $X_3$ for Scenarios 1-2, and Features $Avg(X_1, X_2)$ and $X_3$ for Scenarios 3-4. For all scenarios, different values of $corr(X_1,X_2)=\rho$ were used to assess the impact of feature correlation. Note that $X_3$ serves as a ``control,'' since it is uncorrelated with all other features and has the same importance as $X_1$ (if non-zero) and $X_2$. We summarize the key findings below.

\textbf{Scenario 1}: In truth, $X_1-X_3$ are equally important. For $corr(X_1,X_2)=0$, all methods yielded similar average rankings. As the correlation increases, the ranking of $X_1$ and $X_2$ is inflated and overly optimistic for random forest based permutation importance (PERM-RF). CPI with the correctly specified linear model (CPI-LM) shows ``correct'' rankings for $X_1$ and $X_2$ when $corr(X_1,X_2) < 0.5$, while for $corr(X_1,X_2)\geq 0.5$, the rankings deflate downwards. CPI with random forest (CPI-RF) shows a similar pattern, but likewise inflates the rankings for $corr(X_1,X_2) < 0.5$. As expected, the correctly specified linear model (LM) yields correct rankings for all $corr(X_1,X_2)$ values.

\textbf{Scenario 2}: In truth, $X_1$ is unimportant and $X_2$, $X_3$ are equally important. With the exception of PERM-RF which yielded inflated rankings when $corr(X_1,X_2) \geq 0.5$, all models yielded ``correct'' rankings for $X_1$. For $X_2$, CPI-based rankings were deflated for $corr(X_1,X_2) \geq 0.5$, with some minor deflation seen for PERM-RF.

\textbf{Scenario 3}: In truth, $X_1-X_3$ are equally important but we only observe $Avg(X_1, X_2)$ instead of the individual $X_1$ and $X_2$ features. Compared to LM, the rankings of $Avg(X_1, X_2)$ tend to be slightly smaller when $corr(X_1,X_2)<0.5$ for CPI models and PERM-RF.

\textbf{Scenario 4}: In truth, $X_1$ is unimportant and $X_2$, $X_3$ are equally important, but we only observe $Avg(X_1, X_2)$ instead of the individual $X_1$ and $X_2$ features. Compared to LM, the rankings of $Avg(X_1, X_2)$ are noticeably smaller at lower values of $corr(X_1,X_2)$  for CPI models and PERM-RF, although these rankings increase as the correlation increases.

\begin{figure}[h]
\caption{\textbf{Variable Importance Rankings (Scenario 1 and Scenario 2)}}
\centering
\includegraphics[scale=0.5]{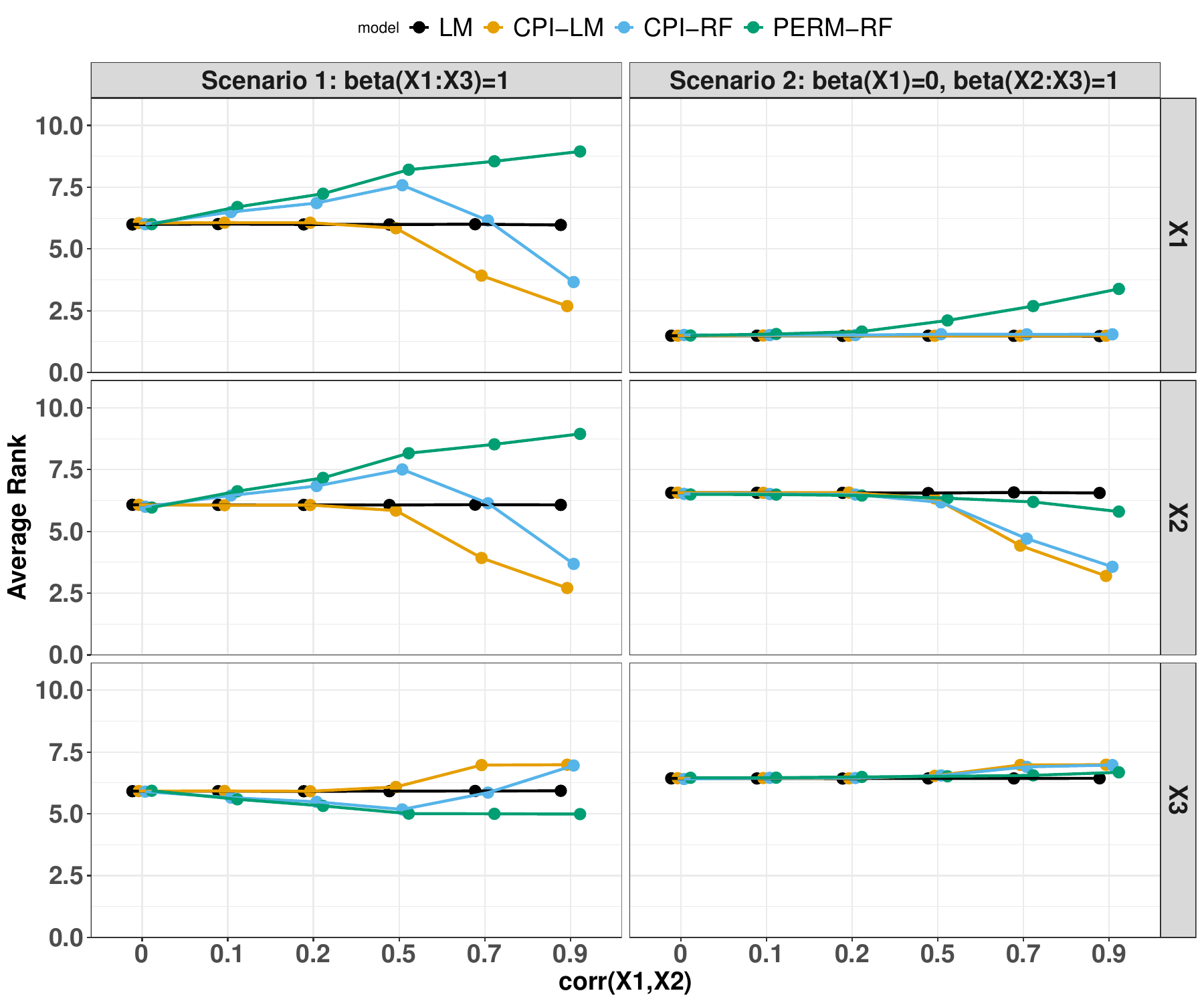}
\label{fig:rank_scen1_scen2}
\end{figure}

\begin{figure}[h]
\caption{\textbf{Variable Importance Rankings (Scenario 3 and Scenario 4)}}
\centering
\includegraphics[scale=0.5]{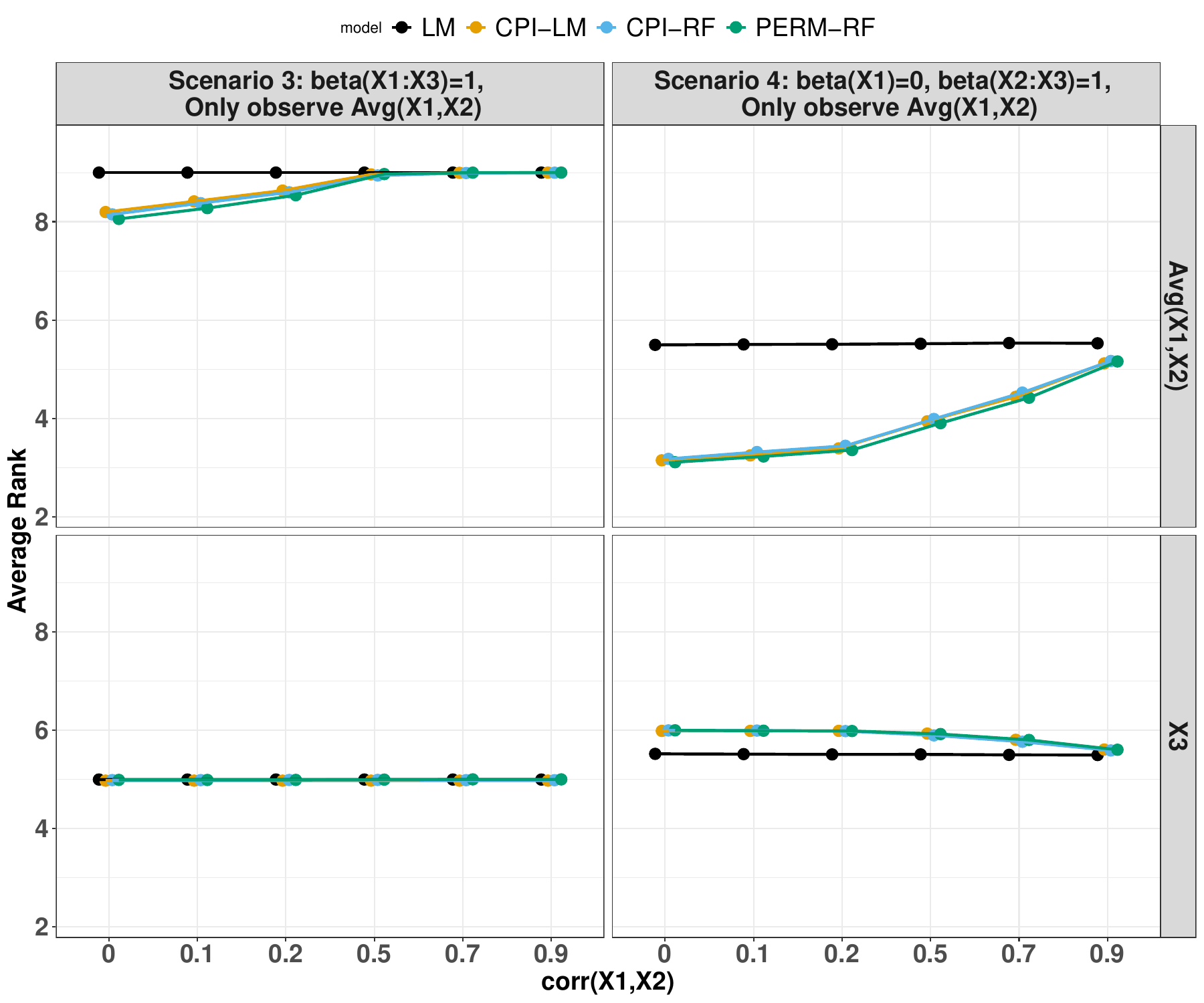}
\label{fig:rank_scen3_scen4}
\end{figure}

To better understand the results above, see Figure \ref{fig:cpi_by_rho}. Here, the mean CPI-based variable importance estimates are plotted against $Corr(X_1, X_2)$ for feature $X_1$ and $X_3$ under Scenario 1; in truth both features are equally important. When the model is specified correctly (cpi lm), the importance estimates for $X_1$ and $X_3$ are equivalent for $Corr(X_1, X_2) \leq 0.50$; the CPI estimate for $X_1$ then drops dramatically for  $Corr(X_1, X_2) > 0.50$. When the model is misspecified (cpi rf), a similar pattern is observed, although the importance estimate for $X_1$ shows some inflation for $Corr(X_1, X_2) < 0.50$.

These results seem to be due to a breakdown in a key knockoff assumption that the knockoffs are independent of the original samples. For Figure \ref{fig:corrX1_tilde_vsCorrX1X2}, $Corr(X_1, \tilde{X}_1)$ [simulation average or theoretical, see theory section] is plotted against $Corr(X_1, X_2)$. We note that for $Corr(X_1, X_2)>0.50$, $Corr(X_1, \tilde{X}_1)$ linearly increases which indicates that the knockoff and observed samples are no longer independent.

\begin{figure}[h]
\caption{\textbf{CPI Estimate vs corr($X_1$, $X_2$)}}
\centering
\includegraphics[scale=0.5]{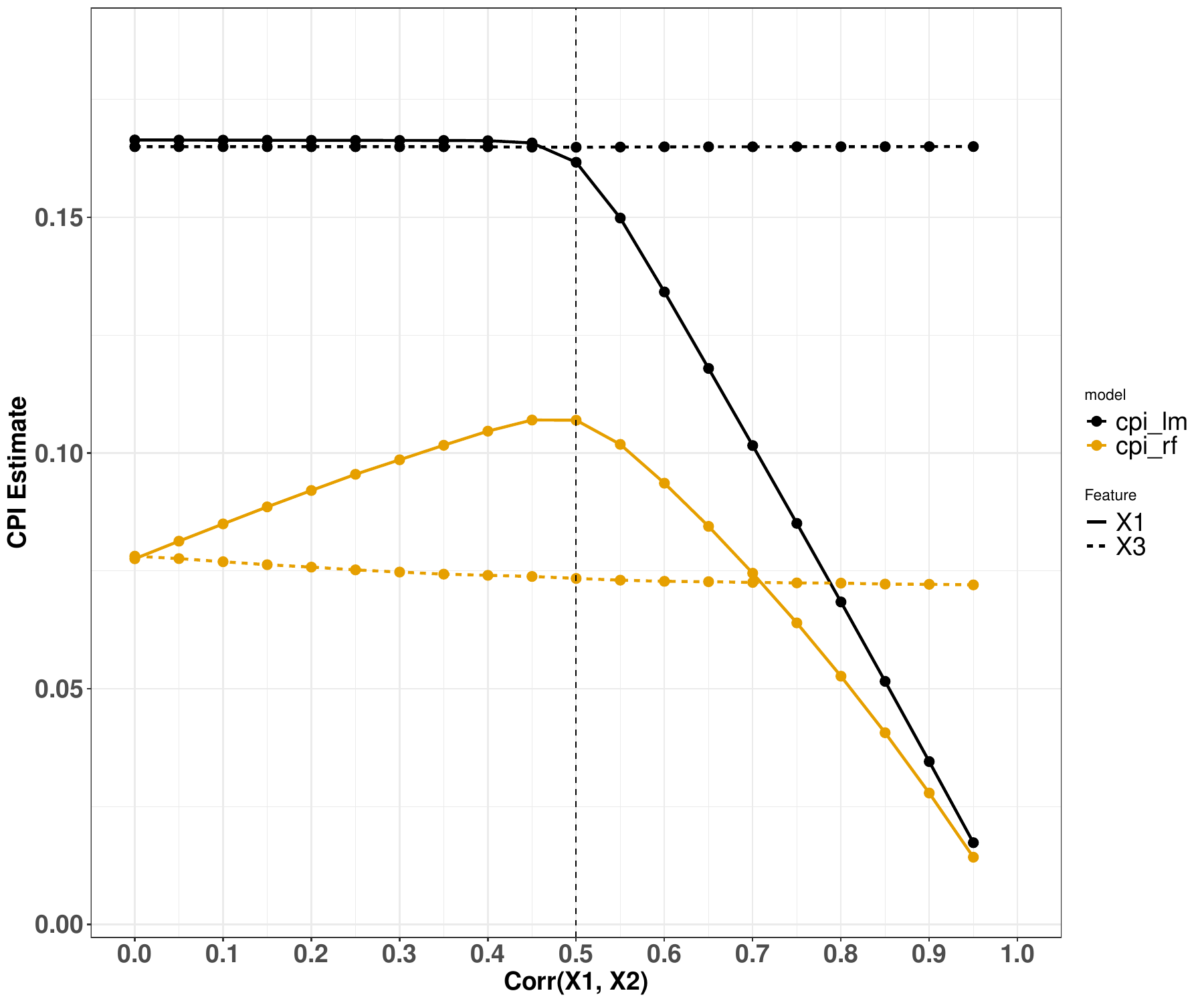}
\label{fig:cpi_by_rho}
\end{figure}

\begin{figure}[h]
\caption{\textbf{Corr($X_1$, $\tilde{X}_1$) vs Corr($X_1$, $X_2$)}}
\centering
\includegraphics[scale=0.5]{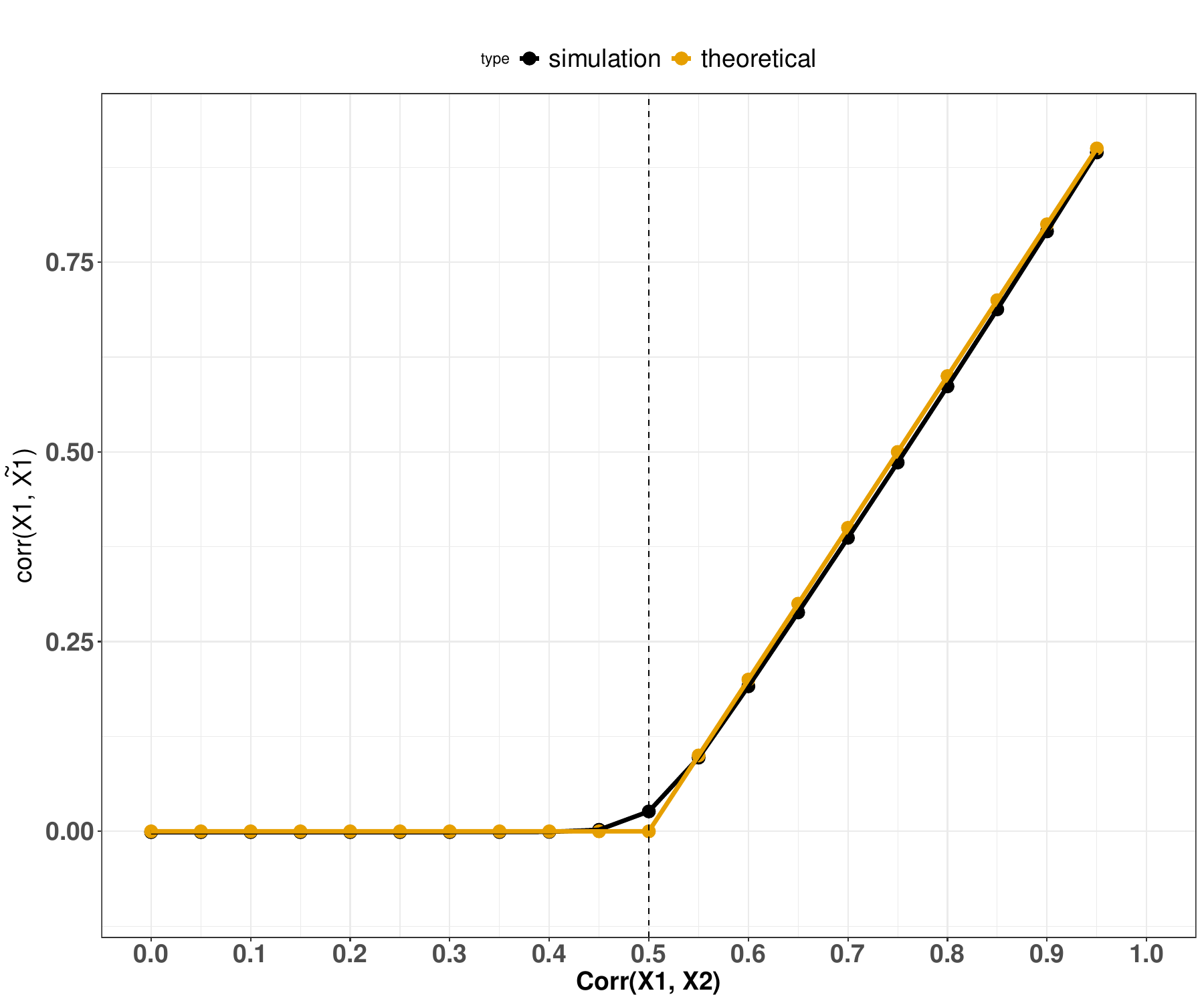}
\label{fig:corrX1_tilde_vsCorrX1X2}
\end{figure}

\clearpage

\section{Theoretical Explanation of the Failure of Knockoff Construction at a Specific Covariate Correlation Threshold}

From the elbow curve in figure \ref{fig:corrX1_tilde_vsCorrX1X2}, we note that in the bivariate case, $X_1$ and $\tilde{X}_1$ are no longer uncorrelated for corr$(X_1, X_2)$ greater than the 0.5 threshold, indicating a failure of the knockoff construction. We theoretically prove this observation below.

\begin{theorem} \label{threshold_theorem}
    \begin{equation*}
        corr(X_1, \tilde{X}_1) =
        \begin{cases}
            0, & \text{if}\ corr(X_1, X_2) \leq 0.5\\
            2 \cdot corr(X_1, X_2) - 1, & \text{otherwise}
        \end{cases}
    \end{equation*}
\end{theorem}

\begin{proof}

We start by constructing the knockoffs, as defined in \cite{candes}.

We set
\[X \sim N(0, \Sigma), \text{ where } \Sigma 
= \begin{pmatrix}
\text{var}(X_1) & \text{cov}(X_1, X_2) \\
\text{cov}(X_1, X_2) & \text{var}(X_2)
\end{pmatrix}
\] and
\[(X, \Tilde{X}) \sim N(0, G), \text{ where } G 
= \begin{pmatrix}
\Sigma & \Sigma-S \\
\Sigma-S & \Sigma
\end{pmatrix} \text{ and } S = \begin{pmatrix}
s_1 & 0 \\
0 & s_2
\end{pmatrix},
\]
so that the relationship between the knockoff variables stay consistent, where $X_i$ and $\tilde{X}_i$ have the same variance, and the covariance between $\tilde{X}_1$ and $\tilde{X}_2$ is the same as the covariance between $X_1$ and $X_2$.

We compute
\begin{equation} \label{eq: x_axis}
    \text{corr}(X_1, X_2) = \frac{\text{cov}(X_1, X_2)}{\sqrt{\text{var}(X_1)}{\sqrt{\text{var}(X_2)}}}
\end{equation}
and
\begin{equation} \label{eq: y_axis}
    \text{corr}(X_1, \tilde{X}_1) = \frac{\text{cov}(X_1, \tilde{X}_1)}{\sqrt{\text{var}(X_1)}{\sqrt{\text{var}(\tilde{X}_1)}}} = \frac{\text{var}(X_1)-s_1}{\text{var}(X_1)}.
\end{equation}

Since the variance of $X_1$ and $X_2$ stay constant when we vary the value of corr$(X_1, X_2)$, the changes in corr$(X_1, \tilde{X}_1)$ are solely due to the changes in $s_1$.

Next, we construct S to ensure that G is positive semi-definite. The values of S along the diagonal are chosen through optimization so that $X_i$ is as orthogonal as possible to $\tilde{X}_i$. Using the Schur complement, we can rewrite $G \succcurlyeq 0$ (G is positive semi-definite) as $2S-S\Sigma^{-1}S^T = 2S - S\Sigma^{-1}S \succcurlyeq 0$ (since S is diagonal). By reversing the Schur complement, this is the same as 
\[\begin{pmatrix}
    \Sigma & S \\
    S & 2S
\end{pmatrix} \succcurlyeq 0\]
And once again, using the Schur complement, this is equivalent to
\[\Sigma - S(2S)^{-1}S = \Sigma - \frac{1}{2}S \succcurlyeq 0 \rightarrow 2\Sigma \succcurlyeq S\]
So, ensuring that $G \succcurlyeq 0$ is equivalent to $2\Sigma \succcurlyeq S$.

To simplify the proof, we assume equal correlation, such that $s_1 = s_2 = s$. Hence, our optimization problem is
\begin{equation*}
    \begin{array}{ll}
	\text{maximize} & s\\
	\text{subject to} & 2\Sigma \succcurlyeq sI, \\
    & s \geq 0,
    \end{array}
\end{equation*}
where $s \in \mathbb{R}$ is the optimization variable.

We can rewrite the constraint $2\Sigma - sI\succcurlyeq 0$ as $\lambda_\text{min}(2\Sigma - sI) \geq 0$, and then $2\lambda_\text{min}(\Sigma) - s \geq 0$. Thus, the optimization problem becomes
\begin{equation*}
    \begin{array}{ll}
	\text{maximize} & s\\
	\text{subject to} & s \leq 2\lambda_\text{min}(\Sigma),\\ 
        & s \geq 0.
    \end{array}
\end{equation*}

Thus, solving the problem, we get that $s^* = 2\lambda_\text{min}(\Sigma)$. Since we also need $\Sigma - S \geq 0$, it must be true that $\text{var}(X_1) - s \geq 0$ and $\text{var}(X_2) - s \geq 0$. So, $s^* = s_1 = s_2 = \min\{2\lambda_\text{min}(\Sigma), \text{var}(X_1), \text{var}(X_2)\}$.

Solving for the minimum eigenvalue, we first normalize $X_1$ and $X_2$ to have mean $0$ and standard deviation $1$. So, the normalized variables $Z_1$ and $Z_2$ have equal variance. We then get that $\lambda_\text{min} = \min u^T\Sigma u$, where $u$ is a unit vector.

Letting $b = \text{cov}(Z_1, Z_2)$, we get that $u^T\Sigma u = 1 + 2bu_1u_2$. Since we are minimizing with respect to $u$, minimizing $1 + 2bu_1u_2$ is equivalent to minimizing $u_1u_2$ such that $u$ is a unit vector.

Solving this minimization problem, we get that 
$(u_1^*, u_2^*) = (\frac{1}{\sqrt{2}}, -\frac{1}{\sqrt{2}})$ or $(-\frac{1}{\sqrt{2}}, \frac{1}{\sqrt{2}})$. Thus, $\lambda_\text{min} = 1 - |b|$.

Therefore,
\[\min\{2\lambda_\text{min}(\Sigma), \text{var}(Z_1), \text{var}(Z_2)\} = \min \{2(1 - |b|), 1\}.\]

We can rewrite \eqref{eq: x_axis} and \eqref{eq: y_axis} using $Z_1$, $Z_2$, and $b$:
\[\text{corr}(Z_1, Z_2) = \frac{\text{cov}(Z_1, Z_2)}{\sqrt{\text{var}(Z_1)}{\sqrt{\text{var}(Z_2)}}} = b,\]
and
\[\text{corr}(Z_1, \tilde{Z}_1) =  \frac{\text{var}(Z_1)-\text{min}\{2(1 - |b|)\}}{\text{var}(Z_1)} = 1 - \text{min}\{2(1 - |b|)\}.\]

Considering $b > 0$, we have two cases for $\min\{2(1 - b), 1\}$:
\begin{align*}
    \text{Case 1 } \left(b \leq \frac{1}{2}\right) \text{: } & 2\left(1 - b\right) \geq 2\left(\frac{1}{2}\right) = 1 \rightarrow \min\{2(1 - b), 1\} = 1,\\
    \text{Case 2 } \left(b > \frac{1}{2}\right) \text{: } & 2\left(1 - b\right) < 2\left(\frac{1}{2}\right) = 1 \rightarrow \min\{2(1 - b), 1\} = 2\left(1-b\right).
\end{align*}

Thus, in case 1, corr$(Z_1, \tilde{Z}_1) = 1 - \text{min}\{2(1 - |b|)\} = 1 - 1 = 0$. In case 2, corr$(Z_1, \tilde{Z}_1) = 1 - \text{min}\{2(1 - |b|)\} = 1 - 2(1-b) = 2b - 1 = 2\cdot\text{corr}(Z_1, Z_2) - 1$.

Transforming the normalized variables $Z_1$ and $Z_2$ back to $X_1$ and $X_2$, we get that in case 1, corr$(X_1, \tilde{X}_1) = 0$, and in case 2, corr$(X_1, \tilde{X}_1) = 2\cdot\text{corr}(X_1, X_2) - 1$.

In conclusion, we have proven Theorem \eqref{threshold_theorem}:

\begin{equation*}
    corr(X_1, \tilde{X}_1) =
    \begin{cases}
        0, & \text{if}\ corr(X_1, X_2) \leq 0.5\\
        2 \cdot corr(X_1, X_2) - 1, & \text{otherwise}.
    \end{cases}
\end{equation*}

\end{proof}

\section{Conclusion}

In this paper, we confirmed the limitations of traditional variable importance measures, specifically the inflation of variable importance rankings in the presence of feature correlation. We then discussed the limitation that highly correlated features pose for the conditional predictive impact (CPI), a variable importance measure based on the knockoff generation to deal with feature correlation. Specifically, we theoretically proved that in the bivariate case, the correlation between the knockoff variable and its corresponding predictor variable is zero, as expected by construction, only when the correlation between the two predictor variables is less than 0.5; however, it increases linearly as the correlation between the two predictor variables exceeds 0.5.

This emphasizes the importance of understanding the utility and limitations behind each method, selecting the optimal measure based on the problem and data at hand. In fields that emphasize interpretability, permutation importance can perform well to counteract the black box nature of some models, such as random forests, if the predictors are uncorrelated (unlikely in practice). However, there may be extrapolation bias in the measure due to the feature permutation placing higher weights in sparse regions. Conditional adjustments to such methods can be used in cases of low to moderate correlation between the features. As proven in this paper, we caution that results can be misleading if there is moderate to strong correlation between features.

Thus, more work needs to be done in the field of interpretable machine learning to completely solve the challenge of between feature correlation for estimating variable importance. Potential options include clustering correlated variables and selecting one per cluster or aggregating correlated variables. Besides feature correlation, additional work includes examining the effect of interactions between variables for variable importance estimation. Overall, there is no free lunch when it comes to high feature correlation, and we should approach new methodologies with a critical mindset, thoroughly understanding their uses and limitations.

\newpage
\bibliography{sources}

\end{document}